\pdfoutput=1

\documentclass[oneside]{article}

\usepackage{amsthm, amssymb, amsmath}
\usepackage{xcolor}
\usepackage[algo2e,ruled]{algorithm2e}
\usepackage{graphicx}
\usepackage{caption}
\usepackage{subcaption}
\usepackage{dsfont}
\usepackage{flexisym}
\usepackage{tabu}

\theoremstyle{definition}

\newtheorem{theorem}{Theorem}[section]

\newtheorem{proposition}[theorem]{Proposition}
\newtheorem{definition}[theorem]{Definition}
\newtheorem{lemma}[theorem]{Lemma}
\newtheorem{corollary}[theorem]{Corollary}

\theoremstyle{remark}

\newtheorem{remark}{Remark}[section]

\title{$HS^2$: Active Learning over Hypergraphs}
\author{
I (Eli) Chien, Huozhi Zhou, and Pan Li\\
Department ECE,UIUC\\
Email: $\{$ichien3,hzhou35,panli2$\}$@illinois.edu
}

\begin{document}
\maketitle
%

%




\begin{abstract}
We propose a hypergraph-based active learning scheme which we term $HS^2$; $HS^2$ generalizes the previously reported algorithm $S^2$ originally proposed for graph-based active learning with pointwise queries~\cite{dasarathy2015s2}. Our $HS^2$ method can accommodate hypergraph structures and allows one to ask both pointwise queries and pairwise queries. Based on a novel parametric system particularly designed for hypergraphs, we derive theoretical results on the query complexity of $HS^2$ for the above described generalized settings. Both the theoretical and empirical results show that $HS^2$ requires a significantly fewer number of queries than $S^2$ when one uses $S^2$ over a graph obtained from the corresponding hypergraph via clique expansion.








\end{abstract}

\section{Introduction}
Active learning is useful for many machine learning applications where the acquisition of labeled data is expensive and time consuming~\cite{cohn1994improving}. In this setting, 
the learner aims to query for as few labels of data points as possible while still guaranteeing a desired level of label prediction accuracy. 

Graph-based active learning (GAL) refers to the particular case when graphs can be used to represent the data points. Such graphs may either come from real-life networks, e.g.\ social networks~\cite{bilgic2010active}, or some transformation based on data points, e.g.\ nearest neighbor graphs~\cite{he2008nearest}. 
Due to the prevalence of graph-structured data, GAL has attracted significant attention in the recent research literature. Most previous works on GAL shared a similar approach: the nodes selected for labeling are determined by minimizing some assumptive empirical error~\cite{zhu2003combining} or upper bound of an empirical error~\cite{guillory2009label,gu2012towards,cesa2013active}. Recently, Dasarathy et al.\ studied the GAL problem from a substantially different angle~\cite{dasarathy2015s2}: They attempted to directly detect the boundaries of different classes over the graph, which further lead to the classification of nodes. Their  approach, termed $S^2$,  benefits  from  tracking  labels  of  the  midpoint of  the shortest path among all paths whose ending nodes are with different labels. 
Surprisingly, $S^2$ is shown to be nearly min-max optimal for the non-parametric setting of active learning problems~\cite{dasarathy2015s2}. 

All above works for GAL depend on a \emph{pointwise oracle}, that is, each response of a query leading to the label of one vertex. However, recent works have pointed out that humans are better at making comparison, and therefore, a \emph{pairwise oracle}, whose response is of the form ``nodes $u$ and $v$ (do not) belong to in the same class'', appears more practical than the pointwise oracle ~\cite{ashtiani2016clustering,mazumdar2017clustering,chien2018query,tsourakakis2017predicting}. Mazumadar and Saha~\cite{mazumdar2017query} proposed a GAL algorithm that uses pairwise oracles. However, their algorithm strongly depends on the assumption that the underlying graph is generated by a stochastic block model (SBM)~\cite{holland1983stochastic}.  

    Graphs essentially capture pairwise relations between data points. Recently, many machine learning and data mining applications have found that hypergraphs modeling high-order relations may lead to better learning performance than traditional graph-based models~\cite{zhou2007learning,agarwal2006higher,benson2016higher}. For example, in subspace clustering, a fit to $d$-dimensional subspace can only be evaluated over at least $d+1$ data points~\cite{agarwal2005beyond,li2017inhomogeneous}; in hierarchical species classification over a foodweb, a carbon-flow unit based on four species appears to be most predictive~\cite{li2017inhomogeneous}. Although some unsupervised or semisupervised learning algorithms over hypergraphs have been reported in ~\cite{chien2018community, chien2018minimax, li2018quadratic}, few works targeted the setting of hypergraph-based active learning (HAL). Intuitively, to solve HAL problems, one may first ``project'' hypergraphs into graphs by replacing hyperedges with cliques (typically this is a procedure termed clique expansion (CE)~\cite{agarwal2005beyond,zhou2007learning}) and then use traditional GAL methods. However, a large number of works have demonstrated that CE causes distortion and leads to undesired learning performance~\cite{hein2013total,li2018submodular}. To the best of our knowledge, the approach proposed by Guillory et al.~\cite{guillory2011active} is the only work that may directly handle HAL. This method follows the traditional rule of GAL~\cite{guillory2009label} that attempts to minimize an upper bound one the prediction error and only works for a pointwise oracle. 


Here, we focus on HAL problems. Following the rule used in $S^2$~\cite{dasarathy2015s2}, we develop several active learning algorithms, termed with the prefix ``$HS^2$'' , which are compatible with multiple classes and with pointwise/pairwise oracles in the hypergraph setting. 
Our contributions are as follows. First, for the setting with a pointwise oracle, we provide a more general algorithm with tighter analysis compared to \cite{dasarathy2015s2}: We consider the $k$-class ($k\geq2$) setting instead of the two-class one considered in the original $S^2$ algorithm. We define novel complexity parameters that can handle the complex interaction between $k$-ary labels and hyperedges. We derive a tighter bound of query complexity, which, as a by-product, justifies that the proposed algorithm indeed requires fewer queries than a simple combination of CE and $S^2$. Second, for the setting with a pairwise oracle, we develop the first model-free HAL/GAL algorithm: Our algorithm does not need any generative assuptions on the graphs/hypergraphs like SBM~\cite{mazumdar2017query}; Moreover, the corresponding analysis for the pairwise-oracle setting is novel.

The paper is organized as follows: 
In Section \ref{sec:probform}, we introduce the notation and our problem formulation. In Section \ref{sec:hypergraph}, we focus on the case of the pointwise oracle and theoretically demonstrate the superiority of $HS^2$ over a combination of CE and $S^2$. 
In Section \ref{sec:pairwiseO} we focus on the case of the pairwise oracle. In Section \ref{sec:exp} we present experiments for both synthetic and real-world data to verify our theoretical findings. Due to the page limitation, we defer the missing proofs to the supplement.
\section{Problem formulations}\label{sec:probform}

We use $G = (V,E)$ to denote a hypergraph with a node set $V$ and a hyperedge set $E$. A hyperedge $e\in E$ is a set of nodes $e\subset V$ such that $|e|\geq 2$. When for all $e\in E$, $|e|=2$, $G$ reduces to a graph.

Suppose that each node belongs to one of $k$ classes. Let $[k]$ denote the set $\{1,2,...,k\}$. A \emph{labeling function} is a function $f: V \mapsto[k]$ such that $f(v)$ is the label of node $v$. Given the labels of all nodes, we call a hyperedge $e$ a \emph{cut hyperedge} if there exist $u,v\in e,f(u)\neq f(v)$. The \emph{cut set} $C$ includes all cut hyperedges. Moreover, define the boundary of the cut set $C$ as $\partial C = \bigcup_{e\in C}e$, i.e., the set of nodes that appear in some cut hyperedges. 
By removing all the cut hyperedges, we suppose that $G$ is partitioned into $T$ connected components whose node sets are denoted by $V_1, V_2, ..., V_T$. For any pair of connected components $V_r,\,V_s$, define the associated \emph{cut component} as $C_{rs} = \{e\in C:e \cap V_r\neq \emptyset,e \cap V_s\neq \emptyset\}$. Note that two different cut components of hyperedges $C_{rs}$ and $C_{r's'}$ may have intersection 
in the hypergraph setting and the union of $C_{rs}$ for all $(r,s)$ pairs is the cut set $C$. 

As we are considering active learning problems, in which the learner is allowed to ask queries and collect information from the oracle. In this work, we study two kinds of oracles: the pointwise oracle $\mathcal{F}_0:V\mapsto [k]$ and the pairwise oracle $\mathcal{O}_0:V\times V\mapsto \{0,1\}$, which are defined as follows. For all $v_1, v_2\in V$,
\begin{align*}
    \mathcal{F}_0(v_1)= f(v_1),\;
 \mathcal{O}_0(v_1, v_2)=\begin{cases}1, &\;\text{if}\;f(v_1)=f(v_2)\\
0, &\;\text{if}\;f(v_1)\neq f(v_2)\\
\end{cases}
\end{align*}
In the pairwise setting, we also allow for a noisy oracle, denoted by $\mathcal{O}_p$, where $p$ stands for the error probability of the oracle, i.e.,
\begin{align*}
  P(\mathcal{O}_p(v_1, v_2)=\mathcal{O}_0(v_1, v_2)) = 1-p
\end{align*}
We assume that for different pairs of nodes, the responses of the oracle are mutually independent. However, for each pair of node $(v_1,v_2)$, $\mathcal{O}_p(v_1, v_2)$ is consistently 0 or 1. Therefore, querying one pair multiple times does not lead to different responses or affect the learning performance.


We use the term \emph{query complexity}, denoted by $\mathcal{Q}$, to quantify how many times an algorithm uses the oracle. Our goal is to design algorithms which learn the partition $V=\bigcup_{i=1}^T V_i$, or equivalently cut set $C$, with query complexity $\mathcal{Q}$ as low as possible. In this work, due to the randomness of the proposed algorithms, we focus on learning the exact $C$ with high probability. That is, given a $\delta\in (0,1)$, with probability $1-\delta$ we recover $C$ with query complexity $\mathcal{Q}(\delta)$.\\
\begin{remark}
The original $S^2$ paper considers an simpler noise model~\cite{dasarathy2015s2}, where one allows independent responses after querying for a single event multiple times. In this case, a simple majority voting can be used for aggregating and denoising the information. However, according to real-life experiments in crowdsourcing~\cite{mazumdar2017clustering, prelec2017solution}, such a method intrinsically introduces bias and thus majority voting may even increase the error. Therefore, we consider a more realistic model used in~\cite{mazumdar2017clustering}. Also note that this noise model is not applicable to the case of pointwise oracle as the noise may always lead to some incorrect labels that can never be fixed.  
\end{remark}

\section{$HS^2$ with a pointwise oracle}\label{sec:hypergraph}


In this section, we propose the $HS^2$ algorithm with a pointwise oracle, termed $HS^2$-point, which essentially generalizes the $S^2$ algorithm for GAL~\cite{dasarathy2015s2} to the hypergraph setting. 
$HS^2$-point is similar to $S^2$, in so far that the algorithm only asks for the label of the midpoint of current shortest path among all paths that connect two nodes with different labels, while the path now is defined over hypergraphs.
The novelty of $HS^2$-point appears in the corresponding analysis of the query complexity. We find that how well cut components are clustered determines the  query  complexity. Later, we will formally define it as the \emph{clusteredness} of cut components. 
In contrast to~\cite{dasarathy2015s2}, for $HS^2$-point, clusteredness of cut components is determined by a much more complicated measure that characterizes the distance between cut hyperedges. Moreover, we tighten the original analysis for $S^2$. As a corollary, the tightened bound shows that $HS^2$-point requires lower query complexity than a naive combination of the clique-expansion method and $S^2$.

We start by introducing the $HS^2$-point algorithm. As $HS^2$ depends on shortest paths, we first define a path in hypergraphs and its length.
\begin{definition}[Path in hypergraph]\label{def:hyperpath}
    Given a hypergraph $G = (V,E)$, we say there is a path of length $l$ between nodes $u,v\in V$ if and only if there exists a sequence of hyperedges $(e_1,e_2,...,e_l)\subseteq E$ such that $u\in e_1$, $v\in e_l$ and $e_i \cap e_{i+1} \neq \emptyset \quad\forall i\in[l-1].$
\end{definition}
\begin{algorithm2e}[htb]
\caption{$HS^2$-point}\label{alg:S2}
\SetAlgoLined
\DontPrintSemicolon
\SetKwInOut{Input}{Input}
\SetKwRepeat{Do}{do}{while}
\SetKwInOut{Output}{Output}
\SetKwInput{Mainalgo}{Main Algorithm}
  \Input{A hypergraph $G$, query complexity budget $\mathcal{Q}(\delta)$
  }
  \Output{A partition of $V$
  }
    \Mainalgo{
    $L\leftarrow\emptyset$\;
    \While{1}{
    $x\leftarrow $ Uniformly at random pick an unlabeled node.\;
    \Do{$x\leftarrow MSSP(G,L)$ exists}{
        Add $(x,\mathcal{F}_0(x))$ to $L$\;
        Remove all hyperedges containing nodes with different label from $G$.\;
        \If{more than $\mathcal{Q}(\delta)$ queries are used}{
            Return the remaining connected components of $G$ 
        }
    }
    }
  }
\end{algorithm2e}
\vspace*{-.4cm}

Conceptually, the algorithm operates by alternating two phases: random sampling and aggressive search. Each outer loop corresponds to a random sampling phase, where the algorithm will query randomly. This phase will end when two nodes with different labels are detected and there is a path connecting them, which is determined by the subroutine $MSSP(G,L)$. Then, the algorithm turns into the inner loop, i.e., the aggressive search phase that searches cut hyperedges. In the inner loop, cut hyperedges are gradually removed and $G$ breaks into a collection of connected components. $L$ is a list to collect labeled nodes with labels. Algorithm 1 will keep tracking the size of $L$. When the query complexity budget is used up, the algorithm ends and outputs the remaining connected components of $G$.

The aggressive search phase that finds all cut hyperedges within low query complexity is the most important step. The key idea is the following. On the path between two nodes with different labels, there must be at least one cut hyperedge. Intuitively, to efficiently find this cut hyperedge, we may use a binary-search method along the shortest one of such paths. That is, we iteratively query for the label of the node that bisects this path. 
The binary search and the search of a shortest path are done simultaneously by the key subroutine $MSSP(G,L)$ (Algorithm \ref{alg:MMSP}). Finding the shortest path in the hypergraph can be implemented via standard BFS algorithm~\cite{cormen2009introduction}. A more efficient way to search the shortest path in a dynamic hypergraph is described in \cite{gao2015dynamic}. Since we focus on query complexity, discussion of the time complexity of the algorithms is outside the scope of the paper.\\
\vspace*{-.5cm}
\begin{algorithm2e}[htb]
\caption{MSSP}\label{alg:MMSP}
\SetAlgoLined
\DontPrintSemicolon
\SetKwInOut{Input}{Input}
\SetKwRepeat{Do}{do}{while}
\SetKwInOut{Output}{Output}
\SetKwInput{Mainalgo}{Main Algorithm}
  \Input{The hypergraph graph $G$, label list $L$
  }
  \Output{The midpoint of shortest-shortest path
  }
    \Mainalgo{\;
        \For{each $v,u\in L$ such that $u,v$ has different label}{
        $P_{v,u}\leftarrow $ shortest path between $v,u$ in $G$.\;
        $l_{u,v}\leftarrow $ length of $P_{u,v}$.($=\infty$ if doesn't exist)\;
        }
        $(v^*,u^*) = \arg\min l_{u,v}$\;
        \uIf{$(v^*,u^*)$ exists and $l_{v^*,u^*}\geq2$ }{
            Return the midpoint of $P_{v^*,u^*}$.\;
        }
        \Else{
            Return $\emptyset$\;
        }
    }
\end{algorithm2e}
\vspace*{-.5cm}

To characterize the query complexity of Algorithm \ref{alg:S2} we need to introduce the following concept.


\begin{definition}[Balancedness]\label{def:balancedness}
    We say that $G$ is $\beta$-balanced if $\beta = \min_{i\in[k]}\frac{|V_i|}{n}$.
\end{definition}
\begin{definition}[Distance between cut hyperedges]\label{def:hypercutedgeDist}
    Let $d_{sp}^{G-C}(v,u)$ denote the shortest path between nodes $v,u$ with respect to the hypergraph $G$ after all cut hyperedges are removed. Let $\Omega_i(e) = \{x\in e | x \in V_i\}$. Define the directed distance between cut hyperedges as $\Delta:C\times C\rightarrow\mathbb{N}\cup\{0,\infty\}$:
    for $e_1$, $e_2$ $\in C$,
    \begin{align}\label{def:Deltahyper0}
        &\Delta(e_1,e_2)\nonumber \\
        &= \sup_{(i,j):e_1,e_2\in C_{i,j}}\bigg(\sup_{v_1\in \Omega_i(e_1)}\inf_{ u_1\in \Omega_i(e_2)}d_{sp}^{G-C}(v_1,u_1) \nonumber\\
        &+ \sup_{v_2\in \Omega_j(e_1)}\inf_{u_2 \in \Omega_j(e_2)}d_{sp}^{G-C}(v_2,u_2) + 1\bigg)
    \end{align}
    If $e_1,e_2$ do not belong to a common cut component, let $\Delta(e_1,e_2) = \infty$.
\end{definition}
For $e_1,e_2$ that belong to certain cut component, the metric $\Delta(e_1,e_2)$ characterizes the length of shortest path that contains $e_2$ after we have found and removed $e_1$. With the above distance, we may characterize the clusteredness of cut hyperedges. First, we need to construct a dual directed graph $H_r = (C,\mathcal{E})$ according to the following rule: the nodes of $H_r$ correspond to cut hyperedges of $G$ and for any two nodes $e, e'$, $ee'$ is an arc in $H_r$ if and only if $\Delta(e,e')\leq r$. According to the definition, each cut component $C_{i,j}$ is mapped to a group of nodes in $H_r$. Now, we may define $\kappa$-clusteredness:
\begin{definition}[$\kappa$-clusteredness]\label{def:hyperkappa}
A cut set $C$ is said to be $\kappa$-clustered if for each cut component $C_{i,j}$, the corresponding nodes in $H_{\kappa}$ are strongly connected.
\end{definition}
    $\kappa$-clusteredness indicates the cut hyperedges in one cut component should not be $\kappa$ away from another cut hyperedge. For better understanding, suppose $HS^2$-point has found and removed the cut hyperedge $e_1$. Another hyperedge $e_2$ in the same cut component appears in the shortest path whose endpoints are in $e_1$. This parameter guarantees that $HS^2$-point needs only at most $\lceil\log_2\kappa\rceil$ queries along such a path to find the cut hyperedge $e_2$.
Hence, if the hypergraph has a small $\kappa$, we can efficiently find all the cut hyperedges in $C$ after we find the first one in each cut component in the random sampling phase. Typically $\kappa$ is not large, as $\kappa$ is naturally upper bounded by the diameter of the hypergraph, which, in a small-world situation, is $O(\log n)$ at most~\cite{watts1998collective}.

The novel part of $HS^2$-point is that we propose Definition~\ref{def:hypercutedgeDist} and Definition~\ref{def:hyperkappa}, which properly generalizes the parametric system of $S^2$ \cite{dasarathy2015s2} to hypergraphs and leads to the following theoretical estimation of query complexity. 

\begin{theorem}\label{thm:1}
  Suppose that $G=(V,E)$ is $\beta$-balanced. The cut set $C$ induced from a label function $f$ is $\kappa$-clustered and $m$ non-empty cut components. Then for any $\delta>0$, Algorithm~\ref{alg:S2} will recover $C$ exactly with probability at least $1-\delta$ if $\mathcal{Q}(\delta)$ is larger than
  \begin{align}
  	& \mathcal{Q}^{\ast}(\delta) \triangleq \frac{\log(1/(\beta\delta))}{\log(1/(1-\beta))}+m(\lceil\log_2(n)-\log_2(\kappa)\rceil) \nonumber \\
    &  + \min(|\partial C|,|C|) (\lceil\log_2(\kappa)\rceil+1) \label{eq:qcom}
  \end{align}
\end{theorem}

    Note that Theorem \ref{thm:1} not only generalizes Theorem 3 from \cite{dasarathy2015s2} to the hypergraph case but also provides a tighter result. Specifically, it improves the original term $|\partial C|$ to $\min(|\partial C|,|C|)$. Recall the definitions of $|\partial C|$ and $|C|$. Typically, $|C| < |\partial C|$ corresponds to the case when the number of cut hyperedges is small while the size of each cut hyperedge is large, which may appear in applications that favor large hyperedges~\cite{hein2013total,li2018submodular,purkait2017clustering}. This improvement is also critical for showing that the $HS^2$-point algorithm has lower query complexity than a simple combination of CE and the original $S^2$ algorithm~\cite{dasarathy2015s2}. We will illustrate this point in the next subsection.

\subsection{Comparison with clique expansion}\label{subsec:CE}
Clique expansion (CE) is a frequently used tool for learning tasks over hypergraphs~\cite{zhou2007learning,agarwal2006higher,chien2018minimax,li2017inhomogeneous}. CE refers to the procedure that transforms hypergraphs into graphs by expanding hyperedges into cliques. Based on the graph obained via CE, one may leverage the corresponding graph-based solvers to solve learning tasks over hypergraphs. For HAL, we may choose a similar strategy. Suppose the obtained graph after CE is denoted by $G^{(ce)}=(V^{(ce)}, E^{(ce)})$, so that $V^{(ce)}=V$, and for $u,\,v\in V^{(ce)}$, $uv\in E^{(ce)}$ if and only if $\exists e\in E$ such that $u,v\in e$. In this subsection we will compare the bounds of query complexity of $HS^2$-point evaluated over $G$ and that of $S^2$ evaluated over $G^{(ce)}$.

Suppose $G$ is $\beta$-balanced, with $m$ cut components and the corresponding cut set $C$ is $\kappa$-clustered. In the following proposition, we show that some parameters of $G^{(ce)}$ are the same as those of $G$.
\begin{proposition}\label{prop:C_ce}
    $G^{(ce)}$ is $\beta$-balanced and has exactly $m$ cut components. Let $C^{(ce)}$ be the cut set of $G^{(ce)}$. Then, $C^{(ce)}$ is $\kappa$-clustered and $|\partial C| = |\partial C^{(ce)}|$. However, we always have $\min(|C|,|\partial C|) \leq \min(|C^{(ce)}|,|\partial C^{(ce)}|)$.
\end{proposition}
As graphs are special case of hypergraphs, Theorem~\ref{thm:1} can be used to characterize the query complexity of $S^2$ over $G^{(ce)}$. For this purpose, recall the parameters in Theorem~\ref{thm:1} that determine the query complexity. Combining them with Proposition~\ref{prop:C_ce}, it is clear that the $HS^2$ algorithm often allows for lower query complexity than that of CE plus $S^2$ and such gain comes from the case when $|C|\leq |C^{(ce)}|$. To see the benefit of $HS^2$-point more clearly, consider the example in Figure \ref{fig:explanation1}.
\begin{figure}[t]
\centering
    \includegraphics[trim={3.4cm 12.3cm 20.5cm 2.3cm},clip, width=0.85\columnwidth]{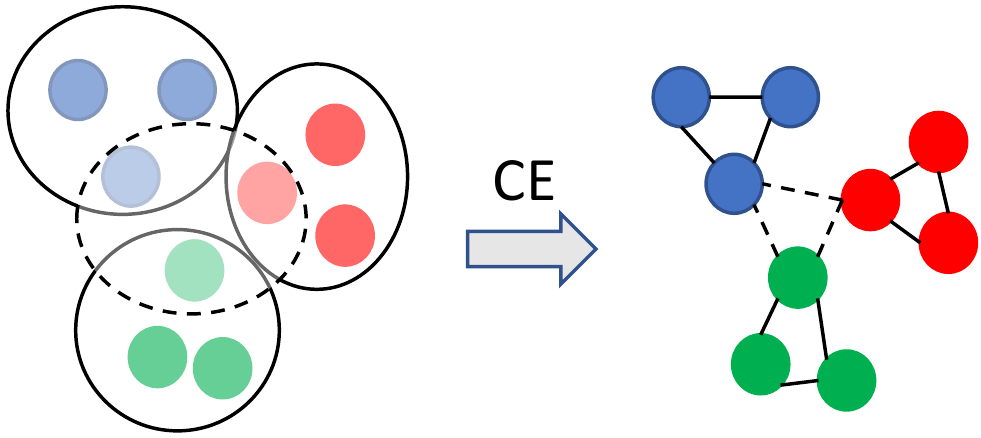}
    \caption{An example of clique expansion. Left: the orginal hypergraph $G$ with 4 hyperedges; Right: the clique-expanded graph $G^{(ce)}$. The colors of nodes identify the labels and the dashed hyperedges/edges are cut hyperedges/edges.}
    \label{fig:explanation1}
\end{figure}
Once $HS^2$-point finds and removes the cut hyperedge of $G$, the correct partition of $V$ is learnt. So we only need to collect the labels of any two nodes in $|\partial C|$. However, if we use $S^2$ over the obtained graph $G^{(ce)}$, all three nodes in $\partial C^{(ce)}(=\partial C)$ must be queried for labels before we learn the correct partition.
\begin{remark}
The benefit of $HS^2$-point essentially comes from the fact that $|C|$ is often smaller than $|C^{(ce)}|$. Note that the query complexity for $S^2$ derived in \cite{dasarathy2015s2} does not reflect such a parametric dependence.
\end{remark}
\begin{remark}
    Note that in the example in Figure~\ref{fig:explanation1} we have $|C|\leq |C^{(ce)}|$ and $|\partial C|=|\partial C^{(ce)}|$. However, $|C|$ is not necessarily smaller than $|C^{(ce)}|$. Consider the following example: Suppose all nodes of $G$ have different labels and there are ${{n}\choose{3}}$ hyperedges in $E$ that cover all triples. Then, $G^{(ce)}$ is a big clique connecting all nodes. In this case $|C| = {{n}\choose{3}} > {{n}\choose{2}} = |C^{(ce)}|$. Nevertheless, in this case we have $|\partial C| = |\partial C^{(ce)}|<|C^{(ce)}|$ and hence Proposition \ref{prop:C_ce} still holds. This example shows that it is non-trivial to prove Proposition~\ref{prop:C_ce}. 
\end{remark}
\section{$HS^2$ with pairwise oracle}\label{sec:pairwiseO}
We now look into the HAL problem with a pairwise oracle. Since the proposed algorithms also depends on the strategy of searching for the shortest path that connects two nodes with different labels, we refer to them as $HS^2$-pair. As mentioned, to our best knowledge, $HS^2$-pair appears to be the first model-free strategy to handle the HAL/GAL problems with a pairwise oracle. 

We analyze settings with both noiseless and noisy oracles. The noiseless case is simple and will be introduced first. Then, we introduce the noisy case that is much more involved.
Note that in the setting with a pairwise oracle, the exact label of each node is not known and not relevant. Hence, without loss of generality, we associate the $i$th class identified during the learning procedure with the label $i$.
\subsection{Noiseless case}
We start by introducing the setting with a noiseless pariwise oracle. The key point is to first seed a few classes and then classify a newly selected node via pairwise comparison with the seeds. If there is a match, we assign the node to its corresponding class; otherwise, we assign the node to a new class. Notationally, we let $S_i, i\in[k]$ be the set of nodes that have been classified to the $i$th class so far. Each $S_i$ starts from one node when a node from the $i$th new class is detected and $S_i$ gradually grows when new nodes of this class are detected. As all nodes $u\in S_i$ share the same label, for a new node $v$, we use $\mathcal{O}_0(v,S_i)$ to denote the query $\mathcal{O}_0(v,u),\,u\in S_i$. The $HS^2$-pair algorithm for the noiseless case is listed in Algorithm \ref{alg:S2_pair}.


\vspace*{-.3cm}
\begin{algorithm2e}
\caption{The noiseless $HS^2$-pair}\label{alg:S2_pair}
\SetAlgoLined
\DontPrintSemicolon
\SetKwInOut{Input}{Input}
\SetKwRepeat{Do}{do}{while}
\SetKwInOut{Output}{Output}
\SetKwInput{Mainalgo}{Main Algorithm}
  \Input{A hypergraph $G$ and query complexity budget $\mathcal{Q}(\delta)$.
  }
  \Output{A partition of $V$.
  }
    \Mainalgo{
    $L\leftarrow\emptyset$, \#c $\leftarrow 1$ \; 
    $v\leftarrow $ Uniformly at random pick an unlabeled node \;\\
    Add $(v,\,1)$ to $L$ and set $S_1 \leftarrow \{x\}$ \; \\
    \While{1}{
    $v\leftarrow $ Uniformly at random pick an unlabeled node \;\\
    \Do{$x\leftarrow MSSP(G,L)$ exists}{
        Collect $\mathcal{O}_0(v,S_i)$ for all $i\in [\#\text{c}]$ \; \\
    	\uIf{$\exists i,\;\mathcal{O}_0(v,S_i) = 1$}{
        Add $(v,\,i)$ to $L$ and $v$ to $S_i$\;
        }\Else{
        $\#\text{c}\leftarrow \#\text{c} + 1$\; \\
        Add $(v,\,\#\text{c})$ to $L$ and Set $S_{\#\text{c}} \leftarrow \{v\}$\;
        }
        Remove all hyperedges containing nodes with different labels from $G$\;
        \If{more than $\mathcal{Q}(\delta)$ queries are used}{
            Return the remaining connected components of $G$
        }
     }
    }
  }
\end{algorithm2e}
\vspace*{-.3cm}
The only difference between $HS^2$-pair in the noiseless case and $HS^2$-point is the way to label a newly selected node. We leverage the pairwise oracle to compare the new node with each class that has been identified. Intuitively, we need at most $k$ pairwise queries to identify the label of each node. Moreover, without additional assumptions on the data, it appears impossible to identify the label of each node with $o(k)$ many pairwise queries. Therefore, combining this observation with Theorem~\ref{thm:1}, we establish the query complexity of Algorithm~\ref{alg:S2_pair} in the following corollary, which essentially is $\Theta(k)$ times the number of queries required by  $HS^2$-point. 
\begin{corollary}\label{cor:pairwise}
 Suppose $G=(V,E)$ is $\beta$-balanced. The cut set $C$ is $\kappa$-clustered and the number of non-empty cut components is $m$. Then for any $\delta>0$, Algorithm~\ref{alg:S2_pair} will recover $C$ exactly with probability at least $1-\delta$ if $\mathcal{Q}(\delta)$ is larger than $kQ^*(\delta)$, i.e.,
    \begin{align*}
          &k\frac{\log(1/(\beta\delta))}{\log(1/(1-\beta))}+km(\lceil\log_2(n)-\log_2(\kappa)\rceil) \\
        & + k\min(|C|,|\partial C|) (\lceil\log_2(\kappa)\rceil+1).
  \end{align*}
\end{corollary}

\subsection{Noisy case}
We consider next the setting with a noisy pairwise oracle. The key idea is similar to the one used in the noiseless case: we first identify seed nodes for the different classes. Due to the noise, however, we need to identify a sufficiently large number of nodes within each class during Phase 1 so that the classification procedure in Phase 2 has high confidence. To achieve this, we adopt a similar strategy as used in Algorithm 2 of~\cite{mazumdar2017clustering} in Phase 1, which can correctly classify a group of vertices into different clusters with high probability based on pairwise queries as long as the size of each cluster is not too small. Phase 2 reduces to classifying the remaining nodes. In contrast to the noiseless case, we adopt a normalized majority voting strategy: we will compare the ratios of the nodes over different classes that claim to have the same label with the incoming node. We list our $HS^2$-pair with noise in Algorithm~\ref{alg:S2_noisypair}. 
\begin{algorithm2e}
\caption{$HS^2$-pair with noise}\label{alg:S2_noisypair}
\SetAlgoLined
\SetKwInOut{Input}{Input}
\SetKwRepeat{Do}{do}{while}
\SetKwInOut{Output}{Output}
\SetKwInput{Mainalgo}{Main Algorithm}
\SetKwInput{Phaseone}{Phase 1}
\SetKwInput{Phasetwo}{Phase 2}
  \Input{A hypergraph $G$, query complexity budget $\mathcal{Q}(\delta)$, parameter $M$
  }
  \Output{A partition of $V$
  }
    \Phaseone{ \\ 
        Uniformly at random sample $M$ nodes from $G$\; \\
        Use Algorithm 2 in \cite{mazumdar2017clustering} on these $M$ nodes to get a partition $S_1,...,S_k$. Let $S = \bigcup_{i=1}^{k}S_i$\;
    }
    \Phasetwo{ \\ 
        $L\leftarrow \{(v, i)| v\in S_i, i\in[k]\}$\; \\
        Remove all hyperedges whose containing different labels from $G$\; 
        \While{1}{
            Uniformly at random sample an unlabeled node $v$\;
            \Do{$x\leftarrow MSSP(G,L)$ exists}{
                $M_i \leftarrow |\{u\in S_i| \mathcal{O}_{p}(u,v) = 1\}|$ for all $i\in[k]$\;
                $i^* \leftarrow \arg\max_{i\in [k]} M_i/|\hat{S}_i|$, add $(v, i^*)$ to $L$\;
                Remove all hyperedges that contain different labels from $G$\;
                \If{more than $\mathcal{Q}(\delta)$ queries are used}{
                     Return the remaining connected components of $G$
                }
            }
        }
    }
\end{algorithm2e}

We know describe the query complexity of Algorithm~\ref{alg:S2_noisypair}.
\begin{theorem}\label{thm:noisycase} 
  Suppose $G=(V,E)$ is $\beta$-balanced. The cut set $C$ induced from a label function $f$ is $\kappa$-clustered and has $m$ non-empty cut components. Then for any $\delta>0,p<\frac{1}{2}$, Algorithm~\ref{alg:S2_noisypair} will recover $C$ exactly with probability at least $1-\delta$ if $\mathcal{Q}(\delta)$ is larger than
  \begin{equation}\label{thm:noisymaineq}
     \mathcal{Q}^{\ast}(\delta/4)M + \frac{128Mk^2\log M}{(2p-1)^4}
  \end{equation}
  where $\mathcal{Q}^{\ast}(\delta)$ is defined in \eqref{eq:qcom},
  and $M$ is an integer satisfying 
  \begin{equation}\label{thm:4req}
    \begin{split}
         & \frac{M}{\log M} \geq \frac{128 k}{\beta (2p-1)^4},\;M \geq \frac{12}{\beta}\log\frac{4k}{\delta},\\
         &\;M \geq \frac{8}{\delta},\;M \geq \frac{2}{\beta D(0.5||p)}\log\frac{8(k-1)\mathcal{Q}^{\ast}(\delta/4)}{\delta}.\\
    \end{split}
  \end{equation}
  Here $D(p||q)$ denotes the KL-divergence of two Bernoulli distributions with parameters $p$ and $q$.
\end{theorem}
We only provide a sketch of the proof of Theorem~\ref{thm:noisycase}. The complete proof is postponed to the supplement.
\begin{proof}
(sketch) In Phase 2, we expect to select $\mathcal{Q}^*(\delta_1)$ nodes for labeling, according to Theorem~\ref{thm:1}. This phase may require $M\mathcal{Q}^*(\delta_1)$ queries. To classify all these nodes correctly via normalized majority voting with probability at least $1-\delta_2$, we require each $S_i$ to be large enough. Specifically, via the Chernoff bound and the union bound, we require 
\begin{align}\label{ineq:1}
 \min_{i\in[k]}|S_i| \geq \frac{1}{D(0.5||p)}\log\frac{2(k-1)\mathcal{Q}^*(\delta_1)}{\delta_2}.
\end{align}
To obtain a sufficiently large $|S_i|$, we need to sample a sufficiently large number of points $M$ in Phase 1. With probability $1-k\text{exp}(-M\beta/8)$, we can ensure that
\begin{align}\label{ineq:2}
\min_{i\in[k]} |S_i|\geq \frac{\beta M}{2}.
\end{align}
Combining~\eqref{ineq:1} and~\eqref{ineq:2} gives the fourth inequality in \eqref{thm:4req}. Moreover, we also need to cluster these $S_i$ correctly via Algorithm 2 in~\cite{mazumdar2017clustering}, which requires the first three constrains in \eqref{thm:4req} and the additional $\frac{128Mk^2\log M}{(2p-1)^4}$ queries according to Theorem 3 in~\cite{mazumdar2017clustering}. This gives the formulas in Theorem~\ref{thm:noisycase}.
\end{proof}
\begin{remark}
    Suppose that the parameters ($p,k,\delta,\beta$) are constants. Then, the fourth requirement of $M$ in \eqref{thm:4req} reduces to $M = O(\log(\mathcal{Q}^{\ast}(\delta)))$, and the overall query complexity equals $O(\mathcal{Q}^{\ast}(\delta)\log(\mathcal{Q}^{\ast}(\delta)))$. Comparing this to Theorem~\ref{thm:1} and Corollary~\ref{cor:pairwise}, we only need      $O(\log(\mathcal{Q}^{\ast}(\delta)))$ times more queries for the setting with the noisy pairwise oracle. 
\end{remark}

Recall the perfect partitioning according to the labels follows $V=\bigcup_{i=1}^T V_i$. If we additionally assume that $T$ equals the number of classes $k$, Phase 1 of Algorithm~\ref{alg:S2_noisypair} will guarantee to sample at least one node from each $V_i,i\in[T]$. This observation allows us to get rid of the random sampling procedure in Phase 2. So the first term in $\mathcal{Q}^*(\delta)$ essentially vanishes. We may achieve the following tighter result.
\begin{corollary}
  Suppose $G=(V,E)$ is $\beta$-balanced. The cut set $C$ induced from a label function $f$ is $\kappa$-clustered and $m$ non-empty cut components. Moreover, suppose $T=k$. Then, for any $\delta>0,p<\frac{1}{2}$, Algorithm~\ref{alg:S2_noisypair} will recover $C$ exactly with probability at least $1-\delta$ if $\mathcal{Q}(\delta)$ is larger than
    \begin{align*}
        \mathcal{Q}_1^{\ast}M + \frac{128Mk^2\log M}{(2p-1)^4}
    \end{align*}
    where
    \vspace{-0.5cm}
    \begin{align*}
        \mathcal{Q}_1^{\ast} & = m(\lceil\log_2(n)-\log_2(\kappa)\rceil)\\
       &  + \min(|\partial C|,|C|) (\lceil\log_2(\kappa)\rceil+1),
    \end{align*}
    and now $M$ is the smallest integer satisfying
    \begin{equation*}
    \begin{split}
        & \frac{M}{\log M} \geq \frac{128 k}{\beta (2p-1)^4},\;M \geq \frac{12}{\beta}\log\frac{3k}{\delta},\\
         &\;M \geq \frac{6}{\delta},\;M \geq \frac{2}{\beta D(0.5||p)}\log\frac{6(k-1)\mathcal{Q}_1^{\ast}}{\delta}.
    \end{split}
  \end{equation*}
\end{corollary}

    In the end of this section, we remark on the CE method in the setting with the pairwise oracle. One still may first apply CE to obtain a graph $G^{(ce)}$ and then run Algorithms \ref{alg:S2_pair} and \ref{alg:S2_noisypair} over $G^{(ce)}$. Corollary~\ref{cor:pairwise} and Theorem~\ref{thm:noisycase} again indicate, the query complexity depends on $\min\{|C|, |\partial C|\}$; By using Proposition \ref{prop:C_ce}, we can again demonstrate the superiority of our proposed approaches over CE-based methods. 

\section{Experiments}\label{sec:exp}
In this section, we evaluate the proposed $HS^2$-based algorithms on both synthetic data and real data. We mostly focus on demonstrating the benefit of $HS^2$ in handling the high-order structures. For the setting with a pointwise oracle, we compare $HS^2$-point with some GAL algorithms including the original $S^2$~\cite{dasarathy2015s2} and EBM~\cite{gu2012towards}, a greedy GAL algorithm based on error bound minimization. To make these GAL algorithms applicable to our high-order data, we first transform hypergraphs into standard graphs by clique expansion which was introduced in Section~\ref{subsec:CE}. For the setting with a pairwise oracle, as there are no other model-free algorithms even for GAL to the best of our knowledge,
we compare $HS^2$-pair over hypergraphs with the combination of clique expansion and $HS^2$-pair over graphs (termed CE + $S^2$-pair later).  All the results are summarized over 100 times independent tests. 

\subsection{Synthetic data}
For the synthetic data experiments, we investigate the effects of the scale of hypergraphs $n$, the number of classes $k$ on all proposed algorithms. We generate labeled hypergraphs according to stochastic block models. The definition of stochastic block models for hypergraphs (HSBM) is not unique, and we adopt the one proposed in \cite{kim2018stochastic} while we allow the number of clusters can be greater than two. Specifically, we fix the sizes of all hyperedges to be 3, which leads to 3-uniform hypergraphs. In addition, we also restrict each cluster to be equal-sized. For each inner-cluster triple of nodes, we generate a hyperedge with probability 0.8. For each intra-cluster triple of nodes, we generate a hyperedge with probability 0.2. 

For the experiments on pointwise queries, 
the results are shown in Figure~\ref{simul:pointwise}. As it shows, the $HS^2$-point outperforms the original $S^2$ and EBM with nontrivial gains. Here, we merely use 3-uniform hypergraphs. If we use hypergraphs with larger hyperedges, the gain of our algorithms will be even greater. In general, the $HS^2$-point will perform better than the other two algorithms if $\min(|C|, |\partial C|)$ becomes much smaller than $\min(|C^{(ce)}|, |\partial C^{(ce)}|)$. What is out of our expectation is the almost linear relation between the query complexity and the scale $n$. This is caused by our particular setting to generate hypergraphs that makes the query complexity be dominated by the last term in~\eqref{eq:qcom}.

\begin{figure}[h]
        \includegraphics[trim={0cm 0cm 1.5cm 1.3cm},clip, width=0.49\columnwidth]{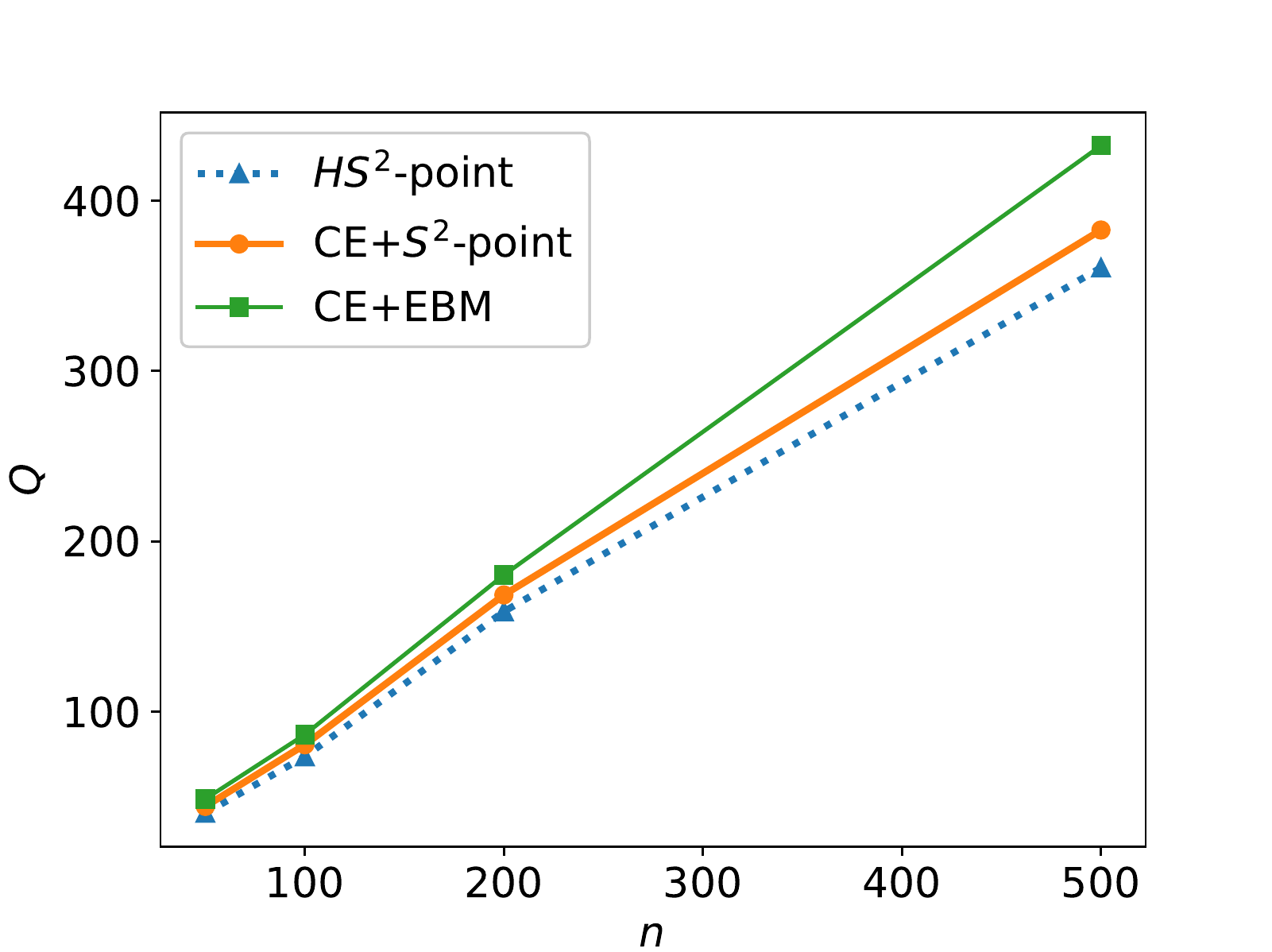}
        \includegraphics[trim={0cm 0cm 1.5cm 1.3cm},clip, width=0.49\columnwidth]{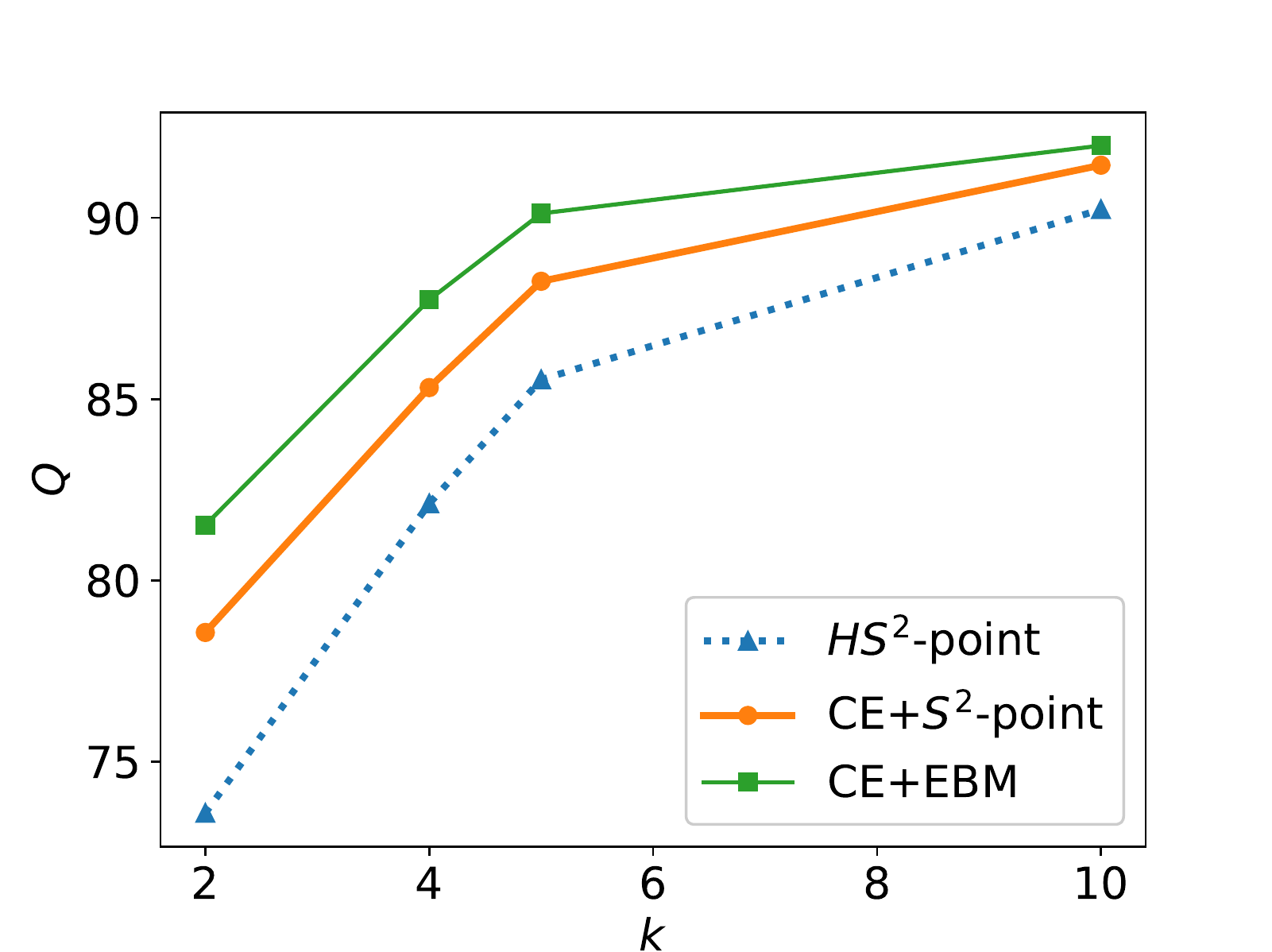}
    \caption{Simulation results with pointwise oracles over synthetic data. Left: query complexity vs scale of hypergraphs $n$ with fixed $k=2$; Right: query complexity vs the number of classes $k$ with fixed $n=100$.}
    \label{simul:pointwise}
    \vspace{-0.2cm}
\end{figure}
For the experiments on pairwise queries, due to the page limitation, we only show the results for $HS^2$-pair in the noiseless case (Algorithm~\ref{alg:S2_pair}). In the real data part, we will evaluate $HS^2$-pair with both noisy and noiseless oracles. The results are described in Figure \ref{simul:pairwise}. Again, our $HS^2$-pair algorithm works better than the naive combination of CE and $S^2$-pair. Also, different from the pointwise case in Figure~\ref{simul:pointwise}, we can see that the query complexity increases almost linearly in the number of classes. This is because we need to use almost $k$ pairwise queries to identify the label of one node.
\begin{figure}[h]
        \includegraphics[trim={0cm 0cm 1.5cm 1.3cm},clip, width=0.49\columnwidth]{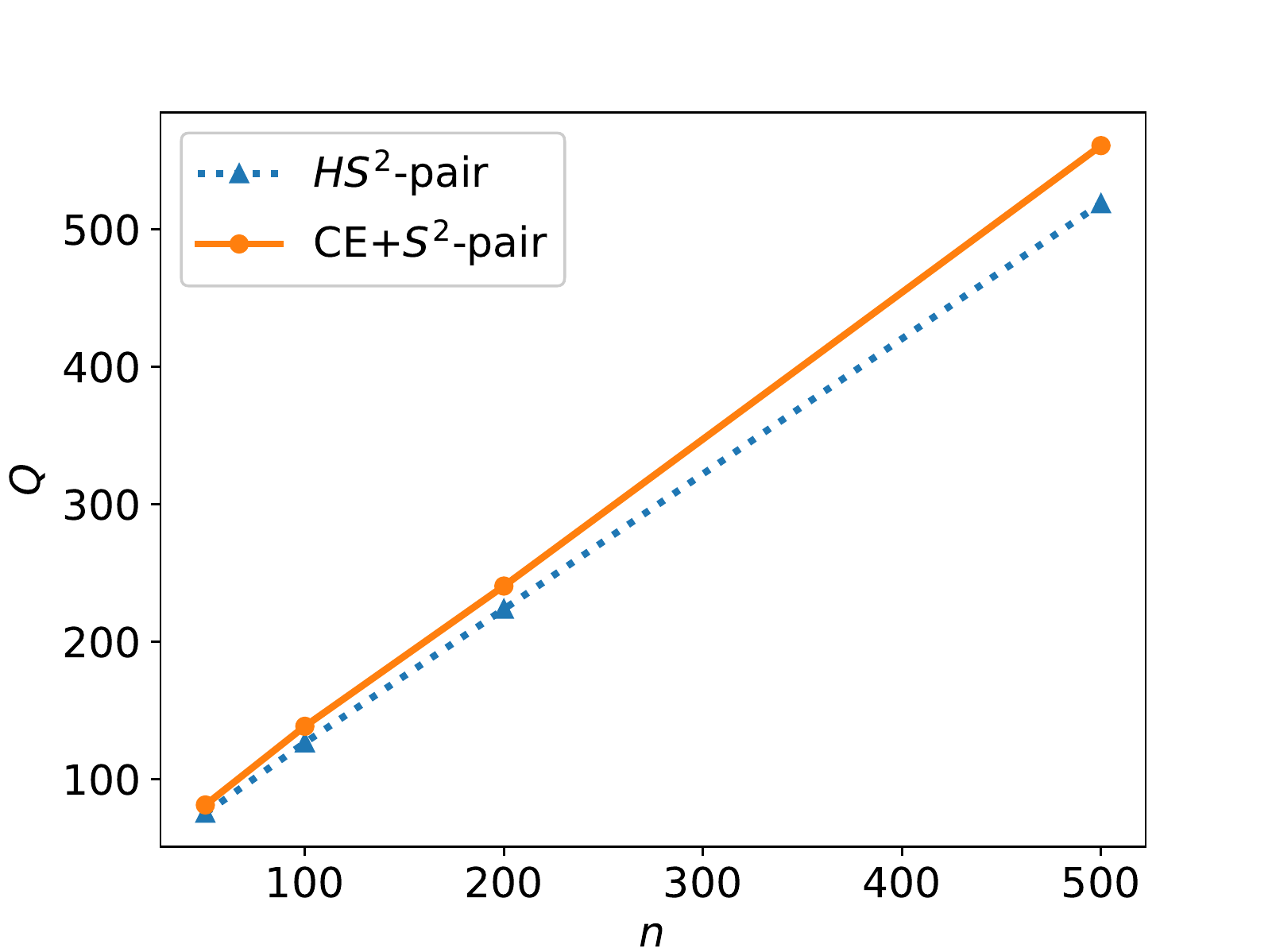}
        \includegraphics[trim={0cm 0cm 1.5cm 1.3cm},clip, width=0.49\columnwidth]{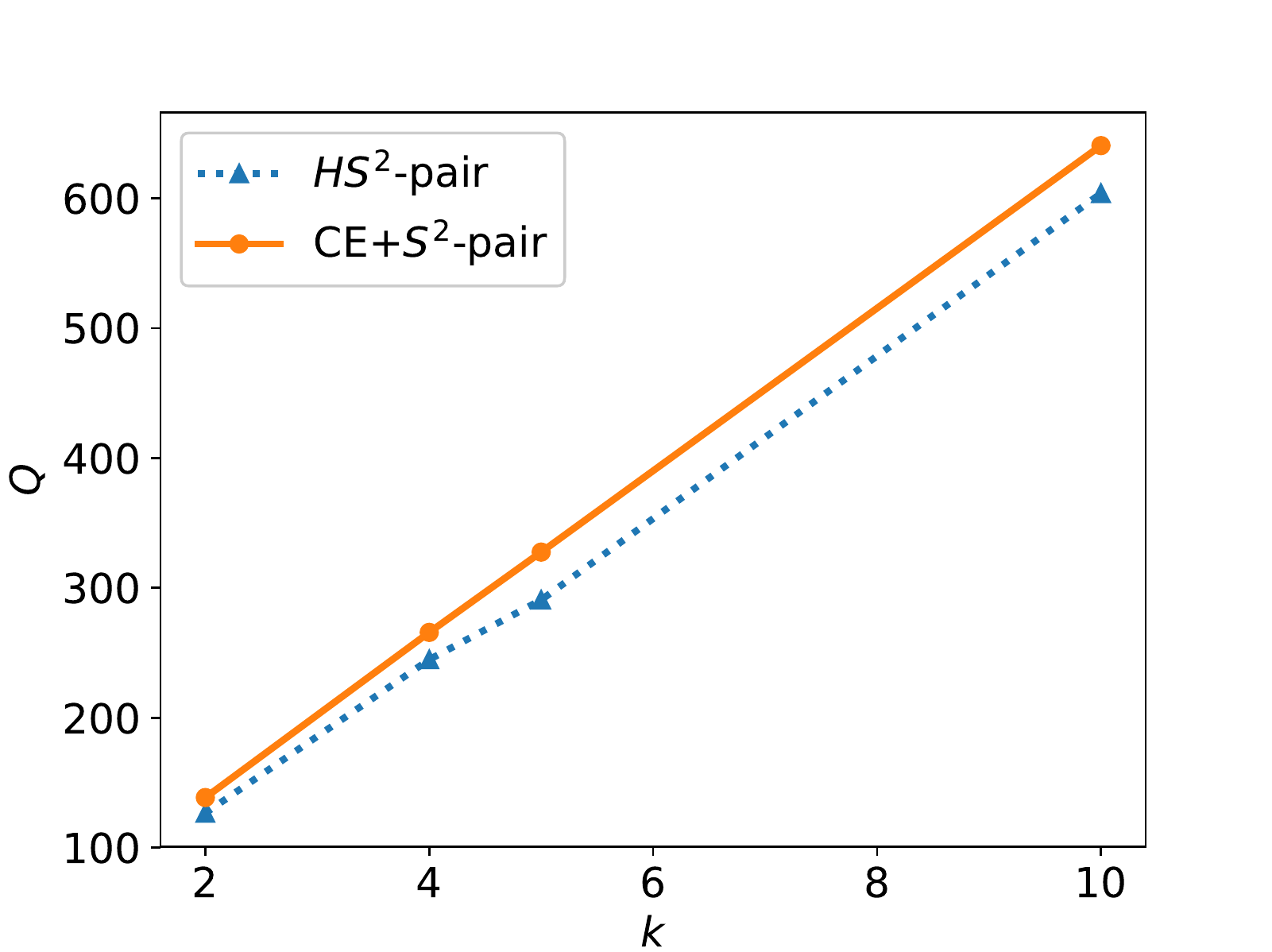}
    \caption{Simulation results with pairwise oracles over synthetic data. Left: query complexity vs scale of hypergraphs $n$ with fixed $k=2$; Right: query complexity vs the number of classes $k$ with fixed $n=100$.}
    \label{simul:pairwise}
    \vspace{-0.5cm}
\end{figure}

\subsection{Real Data}
We also test our algorithms on a real dataset CIFAR-100~\cite{krizhevsky2009learning} which is widely used in machine learning research. Each image in this dataset is represented by a 3072-dimensional feature vector associated with a coarse label as well as a fine label. We construct 3-uniform similarity hypergraphs, where each node corresponds to a image and each hyperedge corresponds to a two-nearest-neighbor relation. Specifically, a hyperedge that connects images $i$, $j$ and $k$ if and only if $j$ and $k$ are $i$'s two nearest neighbors. Note that for real data experiments, in each setting, we fix hypergraphs for all the 100 trials once they are generated. The randomness over the 100 trials still occurs during the execution of different algorithms. These repetitious experiments allows us to give a more in-depth and robust analysis of all algorithms.

\textbf{Noiseless cases.} For the noiseless case, we construct two hypergraphs, of which each is based on 500 images sampled uniformly at random from two mega-classes. These two hypergraphs correspond to binary classification ``fish vs flower'' and ``people vs food container'' respectively. The parameters of these hypergraphs as well as their clique-expansion counterparts are shown in Table~\ref{par:small}. Note that in ``fish vs flower'', $\min(|C|, |\partial C|)=\min(|C^{(ce)}|, |\partial C^{(ce)}|)$ while in ``people vs food container'', $\min(|C|, |\partial C|)<\min(|C^{(ce)}|, |\partial C^{(ce)}|)$.
\begin{table}[h]
\captionsetup{width=\columnwidth}
\centering
\begin{tabular}{|p{1.5cm}<{\centering}|p{0.5cm}<{\centering}|p{0.3cm}<{\centering}|p{0.5cm}<{\centering}|p{0.5cm}<{\centering}|p{1cm}<{\centering}|p{0.75cm}<{\centering}|}
\hline
     & $n$ & $k$  & $|\partial C|$ & $|C|$ & $|\partial C^{(ce)}|$ & $|C^{(ce)}|$\\ \hline
\tiny\shortstack{fish vs\\ flower} & 500 & 2 & 427                     & 435  & 427 & 672\\ \hline
\tiny\shortstack{people vs\\ food container} & 500 & 2 &  412                    &   398 &  412 &  537\\ \hline
\end{tabular}
\caption{Parameters of hypergraphs and their CE counterparts (CIFAR-100, noiseless oracles)}
\label{par:small}
\vspace{-0.5cm}
\end{table}

For the case with pointwise oracles, the results are shown in Table~\ref{real:point}. According to the results, CE can cause information loss and require more queries in general. When comparing the tendency of results between these two hypergraphs, we can see a more significant improvement of $HS^2$-point over the other two in the experiment ``people vs food container'', which further demonstrates our statement on the importance of the parameter $\min(|C|, |\partial C|)$ in Theorem~\ref{thm:1}. 
\begin{table}[h]
\captionsetup{width=\columnwidth}
\centering
\begin{tabular}{|p{1.5cm}<{\centering}|p{1.6cm}<{\centering}|p{2.1cm}<{\centering}|p{1.5cm}<{\centering}|}
\hline
     &  \footnotesize$HS^2$-point  & \footnotesize CE+$S^2$-point & \footnotesize CE+EBM\\ \hline
\tiny\shortstack{fish vs\\ flower} & 438.65 & 445.67 & 460.72\\ \hline
\tiny\shortstack{people vs\\ food container} & 400.89 &421.75 &450.24\\ \hline
\end{tabular}
\caption{Query complexity with the pointwise oracles}
\label{real:point}
\vspace{-0.5cm}
\end{table}

The experimental results of the case with pairwise oracles are shown in Table~\ref{real:pair}. Again, $HS^2$ outperforms the combination of CE and $S^2$.
\begin{table}[htbp!]
\captionsetup{width=\columnwidth}
\centering
\begin{tabular}{|p{1.5cm}<{\centering}|p{2.3cm}<{\centering}|p{2.8cm}<{\centering}|}
\hline
     & \scriptsize noiseless $HS^2$-pair & \scriptsize CE+noiseless $S^2$-pair\\ \hline
\tiny\shortstack{fish vs\\ flower} & 645.72 & 668.39\\ \hline
\tiny\shortstack{people vs\\ food container} & 592.34 &621.51\\ \hline
\end{tabular}
\caption{Query complexity with noiseless pairwise oracles}
\label{real:pair}
\vspace{-0.5cm}
\end{table}

\textbf{Noisy cases.} To test $HS^2$-pair with the noisy oracle, we need to construct larger hypergraphs to illustrate the results, of which each contains $n=5000$ nodes. This is because the number of points required in phase 1 is large, where we set $M$ to be 2000. The parameters of the hypergraphs as well as their CE counterparts are shown in Table~\ref{real:big}.
\begin{table}[htbp!]
\captionsetup{width=\columnwidth}
\centering
\begin{tabular}{|p{1.5cm}<{\centering}|p{0.5cm}<{\centering}|p{0.4cm}<{\centering}|p{0.5cm}<{\centering}|p{0.5cm}<{\centering}|p{0.9cm}<{\centering}|p{0.9cm}<{\centering}|}
\hline
     & $n$  & $k$ &$|\partial C|$ & $|C|$ & $|\partial C^{(ce)}|$ & $|C^{(ce)}|$\\ \hline
\tiny\shortstack{fish vs\\ flower} & \scriptsize 5,000 & \scriptsize 2&\scriptsize4,377                      & \scriptsize4,576  & \scriptsize4,377 &\scriptsize7,314\\ \hline
\tiny\shortstack{people vs\\ food container} & \scriptsize5,000 & \scriptsize 2&\scriptsize4,349                      &  \scriptsize4,214 & \scriptsize4,349 & \scriptsize 6,808\\ \hline
\end{tabular}
\caption{Parameters of hypergraphs and their CE counterparts (CIFAR-100, noisy oracles)}
\label{real:big}
\vspace{-0.5cm}
\end{table}

The results for the case with noisy pairwise oracle are shown in Table \ref{real:faulty}. As expected, $HS^2$ is better than the combination of CE and $S^2$ in terms of query complexity.

\begin{table}[htbp!]
\captionsetup{width=\columnwidth}
\centering
\begin{tabular}{|p{1.5cm}<{\centering}|p{2.3cm}<{\centering}|p{2.7cm}<{\centering}|}
\hline
     &\footnotesize $HS^2$ with noise & \footnotesize  CE+$S^2$ with noise \\ \hline
\tiny\shortstack{fish vs\\ flower} & \footnotesize  8,256,438&\footnotesize 8,754,359 \\ \hline
\tiny\shortstack{people vs\\ food container} &\footnotesize  7,703,549 & \footnotesize 8,487,197\\ \hline
\end{tabular}
\caption{Query complexity with  noisy pairwise oracles}
\label{real:faulty}
\end{table}

\begin{remark}
    According to our theoretical findings, the gains of $HS^2$ over $S^2$+$CE$ depend on how small $|C|$ is compared to $|C_{ce}|$ and $|\partial C|$. Typically, if the given hypergraph consists of a small number of large hyperedges, $|C|$ may become even smaller as opposed to $|C_{ce}|$ and $|\partial C|$. Therefore, in the applications with larger hyperedges, greater gains of $HS^2$ over $S^2$+$CE$ may be expected than the above results that were all obtained over 3-uniform hypergraphs.  Interested readers may refer to the applications in \cite{purkait2017clustering,li2018submodular}.  
\end{remark}




\bibliographystyle{IEEEtran}
\bibliography{Ref}
\clearpage
\section{SUPPLEMENTARY MATERIAL}
\subsection{Proof of Theorem \ref{thm:1}}
We need Lemma 1 of \cite{dasarathy2015s2} which characterizes a bound for the query complexity of the random sampling phase.
We first define a \emph{witness set} of the cut set $C$ as the node set that contains at least one node for each $V_i,i\in[T]$.

\begin{lemma}[Lemma 1 in \cite{dasarathy2015s2}]\label{lma:scqlemma1}
Consider a $\beta$-balancedness graph $G=(V, E)$. For all $\alpha>0$, a subset $W$ chosen uniformly at random is a witness of the cut-set $C$ with probability at least $1-\alpha$ as long as
\begin{align*}
|W|\geq \frac{\log{(\frac{1}{\beta\alpha})}}{\log{(1/(1-\beta))}}
\end{align*}
\end{lemma}
Moreover, we will need the following lemma. Basically it ensures that once $HS^2$-point discovers a cut hyperedge from a cut component, then $HS^2$-point will discover all remaining cut hyperedges in this cut component and the shortest pathes that include these hyperedges are at most with length $\kappa$. 
\begin{lemma}\label{lma:HS2pairNB}
    Suppose a hypergraph $G$ with a cut set $C$ is $\kappa$-clustered. Moreover, suppose $C_{rs}$ is a cut component. If $e\in C_{rs}$ is discovered, which means a pair of nodes $u\in \Omega_r(e),v\in \Omega_s(e)$ are labeled, then at least one remaining cut hyperedge in $C_{rs}$ lies in a path of length at most $\kappa$ from a pair of nodes with labels $r$ and $s$ respectively.
\end{lemma}
\begin{proof}
   By definition \ref{def:hyperkappa}, we know that the hyperedges in $C_{rs}$ will form a strongly connected component in $H_{\kappa}$. This means for any $e\in C_{rs}$, there is at least one $e^{\prime}\in C_{rs}$ such that the arc $ee^{\prime}$ exists in $H_{\kappa}$. Recall there exists arc $ee^{\prime}$ in $H_{\kappa}$ if and only if $\Delta(e,e^{\prime})\leq \kappa$. By definition \ref{def:hypercutedgeDist} this means for any node pair $u\in \Omega_r(e),v\in \Omega_s(e)$, the length of the shortest path including $e^{\prime}$ but excluding $e$ will be less then $\kappa$. Note that in the definition of $\Delta(e,e^{\prime})$, we use the supremum taking over the node set $\Omega_r(e),\Omega_s(e)$. This is because it ensures that no matter which node pair $u,v\in e$ we have, $\Delta(e,e^{\prime})$ can always upper bound the length of the shortest path including $e^{\prime}$ but excluding $e$ with endpoints $u,v$. In contrast, we use the infimum taking over the node set $\Omega_r(e^{\prime}),\Omega_s(e^{\prime})$. This is because it only needs to search for the shortest path. Hence once we find a cut hyperedge $e$ in $C_{rs}$, we are guarantee to find at least one cut hyperedge $e^{\prime}\in C_{rs}$ through a path of length $l_S\leq \kappa$ after we remove $e$.
\end{proof}

Now let us prove Theorem~\ref{thm:1}. The proof uses a similar outline as the proof given in \cite{dasarathy2015s2}. However, we need to take care of the hypergraph structures that are described by the definition \ref{def:hyperpath} to \ref{def:hyperkappa}. We will also derive a tighter bound for the number of runs $R$, which finally yields a lower query complexity than that in \cite{dasarathy2015s2}.

We  can divide the query complexity in two parts which are associated with the random sampling phase and the aggressive search phase respectively. The goal of the random sampling phase is to find a witness set. By applying Lemma \ref{lma:scqlemma1}, we can bound from above the number of random queries needed.

For the aggressive search phase, let $l_{S}(G, L, f)$ be the length of the shortest path among all paths connecting nodes with different labels after we collects the labels of nodes in $L$. After each step of the aggressive search, the $l_{S}(G, L, f)$ will roughly get halved. Thus it will take no more than $\lceil\log_{2}l_{S}(G, L, f)\rceil +1$ steps to find a cut hyperedge. In order to bound the required number of active queries, let us split up the aggressive search phase into ``runs'', where each run ends when a new boundary node has been discovered. Let $R$ be the number of runs, and it's obvious that $R\leq|\partial C|$ since we will at most discovere all the boundary nodes, which is $|\partial C|$. Moreover, we also have $R\leq |C|$. This is because we will discover a new boundary node if and only if we discover at least a cut hyperedge. Hence together we have $R\leq \min(|C|,|\partial C|)$. The observation of $R\leq |C|$ is missed by~\cite{dasarathy2015s2}. However, this part is extremely important for the hypergraph setting according to the later discussion in Remark~\ref{remark:boundR}.

For each $i\in [R]$, let $G^{(i)}$ and $L^{(i)}$ be the graph and label set up to run $i$. Then the total number of active queries can be upper bounded by \\ $\sum_{i=1}^{R}(\lceil\log_2(l_{S}(G^{(i)}, L^{(i)}, f))\rceil+1)$.
By observation, in each run it is trivial $l_{S}(G^{(i)}, L^{(i)}, f)\leq n$. From Lemma \ref{lma:HS2pairNB}, once we discover a cut hyperedge from $C_{rs}$, we're able to find at least one undiscovered cut hyperedge from $C_{rs}$ through a path of length at most $\kappa$ according to Lemma \ref{lma:HS2pairNB}. After running $|C_{rs}|-1$ times, we may fully discover $C_{rs}$.
In all these $|C_{rs}|-1$ runs in $R$, $l_{S}\leq\kappa$. In all, the runs that we first discover each cut components are long runs, whose $l_{S}$ can be upper bounded naively by $n$, and the number of long runs is not greater than $m$. Once we discover the first cut hyperedge in $C_{rs}$, the rest $|C_{rs}|-1$ runs are short runs whose $l_{S}$ can be upper bounded by $\kappa$. 
Therefore, we have
\begin{align*}
&\sum_{i=1}^{R}(\lceil\log_2(l_{S}(G^{(i)}, L^{(i)}, f))\rceil+1)\\
&\leq (R+ m\lceil\log_2 n\rceil +(R-m)\lceil\log_2\kappa\rceil)\\
&\leq m(\lceil\log_2n\rceil-\lceil\log_2\kappa\rceil)+\min(|C|,|\partial C|)(\lceil\log_2\kappa\rceil+1)
\end{align*}
Hence we complete the proof.
\begin{remark}\label{remark:boundR}
Note that in \cite{dasarathy2015s2} they only use the bound $R\leq |\partial C|$ and miss the bound $R\leq |C|$. As they focused on standard graphs, $|C|$ can be lower bounded by $\frac{|\partial C|}{2}$. Therefore, in standard graphs, the bound $R\leq |\partial C|$ will at most loose by a constant factor of $2$. However, in hypergraphs, it's possible that $|C|$ is substantially smaller than $|\partial C|$, when the sizes of hyperedges are large. 
So $R\leq |C|$ is crucial for the tight analysis in the hypergraph scenario.
\end{remark}

\subsection{Proof of Proposition \ref{prop:C_ce}}
    \subsubsection{checking the equal parameters}\label{subsec:checkeq}
    We check the parameters one by one. 
    
    We start from proving that if $C$ is $\kappa$-clustered, then $C^{(ce)}$ is also $\kappa$-clustered.
    We note that performing CE does not change the length of the shortest path of arbitrary node pair $v_1,v_2\in V$. This is because CE will replace a hyperedge by a clique, which makes all nodes in the hyperedge become fully connected. Hence the $C^{(ce)}$ is still $\kappa$-clustered.

    Now, we prove that if $G$ has $m$ non-empty cut components, then $G^{(ce)}$ will also have $m$ non-empty cut components.
    We note that for any non-empty cut component $C_{i,j}$ in $G$, there is at least one hyperedge $e\in C_{i,j}$. By definition, we know that $e\cap V_i \neq \emptyset$ and $e\cap V_j \neq \emptyset$. So after CE, in the clique corresponding to this hyperedge $e$, there must be at least one edge such that one of its endpoint is from $V_i$ and the other one is from $V_j$, which makes $C_{i,j}$ still non-empty in $G^{(ce)}$. On the other hand, for arbitrary $i,j$, the cut component $C_{i,j}$ is empty in $G$ if and only if there is no hyperedge between $V_i,V_j$. Hence $C_{i,j}$ will still be empty in $G^{(ce)}$. Together we show that if there are $m$ non-empty cut components in $G$, there are exactly $m$ non-empty cut components in $G^{(ce)}$.
    
    It is easy to see $G^{(ce)}$ keep $\beta$-balanced as $f$ does not change in CE. 
    
    
    Now, we prove that $|\partial C| = |\partial C^{(ce)}|$. For any $e \in C$, let's denote $e = \{v_1,...,v_d\}$. By definition we know that $v_1,...,v_d \in \partial C$. Suppose $e \in C$ and the nodes $v_1,...,v_d$ can be partitioned into $t$ non-empty set $S_1,...,S_t$ according to their labels. Without loss of generality, let $v_1\in S_1$. Then after CE of $e$ we know that the edges $(v_1,v),v\in S_j,j\in\{2,3,...,t\}$ will be in the set $C^{(ce)}$. By definition of $C^{(ce)}$, we know that all $v\in S_j,j\in\{2,3,...,t\}$ will be in the cut set $\partial C^{(ce)}$. We can repeat the same argument for all nodes in $S_1$ and know that $S_1\subset \partial C^{(ce)}$. In the end, we can show that $\forall v\in e$, $v\in \partial C^{(ce)}$. By definition we also have $\forall v\in e$, $v\in \partial C$. Therefore, we claim that $\partial C = \partial C^{(ce)}$ which furthermore $|\partial C| = |\partial C^{(ce)}|$.
    
    \subsubsection{proof for the inequality}
    
    Now, we prove that $\min(|C|,|\partial C|)\leq \min(|C^{(ce)},|\partial C^{(ce)}|)$. As above, we have proved $|\partial C| = |\partial C^{(ce)}|$. The case when $|\partial C^{(ce)}|\leq|C^{(ce)}|$ is an easy case. So, we only need to prove for the case when $|\partial C^{(ce)}|>|C^{(ce)}|$. We claim that if $|\partial C^{(ce)}|>|C^{(ce)}|$, then $|C|\leq |C^{(ce)}|$, which is proved as follows.
    
    Let us first introduce an auxiliary graph $G^{\prime}$ that can be useful in the proof. $G^{\prime} = (\partial C^{(ce)},C^{(ce)})$ is a subgraph of $G^{(ce)}$ with the node set $\partial C^{(ce)}$ and the edge set $C^{(ce)}$.
    
    In the following, we show that when $|\partial C^{(ce)}|>|C^{(ce)}|$, then it's impossible for $G^{\prime}$ to have any cliques of size greater or equal to 3. Note that by the definition of $C^{(ce)}$ and $\partial C^{(ce)}$, the auxiliary graph $G^{\prime}$ is connected. Moreover, as for the condition $|\partial C^{(ce)}|>|C^{(ce)}|$, we know that the average degree of $G^{\prime}$ is strictly less then $2$. This is because  
    \begin{align*}
    2> \frac{2|C^{(ce)}|}{|\partial C^{(ce)}|} = \frac{\sum_{v\in \partial C^{(ce)}} d_v}{|\partial C^{(ce)}|}
    \end{align*}
    where $d_v$ is the degree of node $v$ in $G^{\prime}$. Hence it's impossible to have any cliques of sizes that are greater than or equal to 3 in $G^{\prime}$. 
    
    By using the above observation and the definition of clique expansion, we know that when $|\partial C^{(ce)}|>|C^{(ce)}|$, all hyperedges in $C$ are actually edges. Equivalently, we have $C = C^{(ce)}$, which implies $|C| = |C^{(ce)}|<|\partial C^{(ce)}|$. This concludes the proof.

    By the end of this subsection, we would like to show that it's possible to have $\min(|C|,|\partial C|)< \min(|C^{(ce)}|,|\partial C^{(ce)}|)$ for some hypergraphs. Let $C$ contain only one hyperedge $e$ such that $|e| = 4$. Then it's obvious to see that $1 = |C|<|C^{(ce)}| = 6$ and $|\partial C^{(ce)}|= |\partial C| = 4$. Hence in this special example we have $\min(|C|,|\partial C|) < \min(|C^{(ce)}|,|\partial C^{(ce)}|)$.

\subsection{Proof of Theorem \ref{thm:noisycase}}
Before we start our proof, we need to prepare preliminary results. 
The first one is Theorem 3 in~\cite{mazumdar2017clustering} that characterizes the theoretical performance of Algorithm 2 in~\cite{mazumdar2017clustering}. 
   \begin{theorem}[Theorem 3 in \cite{mazumdar2017clustering}]\label{thm:arya1}
   Given a set of $M$ points which can be partition into $k$ clusters. 
   The Algorithm 2 in \cite{mazumdar2017clustering} will return all clusters of size at least $\frac{64k\log M}{(1-2p)^4}$ with probability at least $1-\frac{2}{M}$. The corresponding query complexity is $O(\frac{M k^2\log M}{(1-2p)^4})$.
    \end{theorem}
    Basically we use this theorem to analyze Phase 1 of Algorithm~\ref{alg:S2_noisypair}. The next one is a lemma that characterizes a lower bound of the KL divergence of two Bernoulli distributions. 
    \begin{lemma}[\cite{wiki:hoeffding}]\label{lma:KLbound}
    Let us denote $D(x||y)$ be the KL divergence of two Bernoulli distributions with parameters  $x, y \in [0,1]$ respectively. We have
    \begin{equation}
       D(x||y) \geq \frac{(y-x)^2}{2\min\{x,\,y\}}
    \end{equation}
\end{lemma}

\begin{remark}
    Note that the bound is tighter than directly using Pinsker's inequality \cite{pinsker1960information} when $y\leq 1/8$.
\end{remark}
    Now we start to prove Theorem \ref{thm:noisycase}. First we will show that Phase 1 of Algorithm~\ref{alg:S2_noisypair} will return the correct partition $S_1,...,S_k$ with high probability. From Theorem \ref{thm:arya1} we know that we have to ensure our sampled $M$ points contain all underlying true clusters with size at least $O(\frac{M k^2\log M}{(1-2p)^4})$. 
    Since we sample these $M$ points uniformly at random, thus $(S_1,...,S_k)$ is the multivariate hypergeometric random vector with parameters $(n,np_1,...,np_k,M)$ and $\forall i,\;p_i=\frac{|\{v\in V| f(v) = i\}|}{n}$. It's well known (\cite{hoeffding1963probability},\cite{skala2013hypergeometric}) that when $M\leq n/2$, the tail bound for the multivariate hypergeometric distribution is
    \begin{equation}\label{concentration0}
        \begin{split}
             & \mathbb{P}(S_i \leq M(p_i-\frac{p_i}{2}))\leq \exp(-MD(\frac{p_i}{2}||p_i)) \\
             & \leq \exp(\frac{-Mp_i}{8}) \\
             & \Rightarrow \mathbb{P}(S_i \leq \frac{M\beta}{2}) \leq \exp(\frac{-M\beta}{8}),
        \end{split}
    \end{equation}
    where we use Lemma \ref{lma:KLbound} for the second inequality. 
    For the case $M\geq n/2$, we could apply trick of symmetry and have (\cite{hoeffding1963probability},\cite{skala2013hypergeometric},\cite{serfling1974probability})
    \begin{align*}
        & \mathbb{P}(S_i
        \leq M(p_i-\frac{p_i}{2}))\\
        &\leq \exp(-(n-M)D(p_i+\frac{p_iM}{2(n-M)}||p_i)) \\
        & \leq \exp(-(n-M)\frac{(\frac{p_iM}{2(n-M)})^2}{p_i(2+\frac{M}{n-M})})\\
        & = \exp(-\frac{p_iM^2}{4(2n-M)}) \\
        & \leq \exp(-\frac{Mp_i}{12}),
    \end{align*}
    where the second inequality is via Lemma \ref{lma:KLbound} and the last inequality uses the assumption $M\geq n/2$. Hence, for all $M\leq n$, we have
    \begin{equation}\label{concentration1}
        \mathbb{P}(S_i \leq \frac{M\beta}{2}) \leq \exp(\frac{-M\beta}{12})
    \end{equation}
    Since we need \eqref{concentration1} holds for all $i$, we apply the union bound over all $k$ events which gives
\begin{equation}\label{pf:minS}
  \mathbb{P}(\bigcap_{i=1}^{k}\{S_i \geq \frac{M\beta}{2}\}) \geq 1- k\exp(\frac{-M\beta}{12})
\end{equation}
Now, we need $M$ is large enough such that $\frac{M\beta}{2}$ meets the requirement of Theorem~\ref{thm:arya1}. Moreover, we also need $M$ to be large enough such that this event holds with probability at least $1-\frac{\delta}{4}$. For the first requirement, we have
\begin{equation*}
  \frac{M\beta}{2} \geq \frac{64k\log M}{(2p-1)^4} \Rightarrow \frac{M}{\log M} \geq \frac{128k}{\beta (2p-1)^4}
\end{equation*}
This is exactly our first requirement on $M$ in \eqref{thm:4req}. For the high probability requirement, we have
\begin{equation*}
  k\exp(\frac{-M\beta}{12}) \leq \frac{\delta}{4} \Rightarrow M\geq \frac{12}{\beta}\log\frac{4k}{\delta}
\end{equation*}
This is exactly the second requirement on $M$ in \eqref{thm:4req}. Moreover, we also need Algorithm 2 of~\cite{mazumdar2017clustering} successfully recover all the true clusters with probability at least $1-\frac{\delta}{4}$, and thus we have
\begin{equation*}
  \frac{2}{M} \leq \frac{\delta}{4} \Rightarrow M\geq \frac{8}{\delta}
\end{equation*}
This is exactly the third requirement on $M$ in \eqref{thm:4req}. \\
Now assume that Algorithm 2 of \cite{mazumdar2017clustering} indeed returns all true clusters. We will analyze Phase 2. 
Start from assuming all $S_i$'s are correctly clustered. Then for any new node $v$, from the algorithm we designed we will query for comparing $v$ with all the $M$ nodes that have been clustered. Before we continue, let us introduce some error events which is useful for the following analysis. Let $Er^{(i)}$ be the event that a node with label $i$ is incorrectly clustered by the normalized majority voting. Let $Er_j^{(i)} = \{\frac{M_j}{|S_j|} > \frac{M_i}{|S_i|}\}$, for $j\neq i$, where $M_j$ is the number of nodes in  $S_j$ that respond positively to the pairwise comparisons with node $v$. Note that we have $M_j \sim Bin(p,|S_j|)$ for $j\neq i$ and $M_i \sim Bin(1-p,|S_i|)$. All these $M_{l}$'s are mutually independent.\\ 
We start from analyzing the normalized majority voting for the unlabeled node $v$. Then we have
\begin{align*}
    & \mathcal{O}_p(v,u)\sim Ber(1-p)\;\forall u \in S_i;\\
    & \mathcal{O}_p(v,u)\sim Ber(p)\;\forall u \notin S_i
\end{align*}
where we recall that $\mathcal{O}_p(x,y)$ is the query answer for the point pair $(x,y)$ from the noisy oracle $\mathcal{O}_p$. 
 So the error probability $\mathbb{P}(Er^{(i)})$ that we misclassify the point $v$ can be upper bounded by
\begin{equation*}
  \mathbb{P}(Er^{(i)}) \leq (k-1)\max_{j\neq i}\mathbb{P}(Er_j)
\end{equation*}
where we used the union bound.
Moreover, we can upper bound $\mathbb{P}(Er_j^{(i)})$ as following (recall that $p<1/2$)
\begin{align*}
     \mathbb{P}(Er_j^{(i)}) &= \mathbb{P}(\frac{M_j}{|S_j|} > \frac{M_i}{|S_i|}) \\
     &\leq \mathbb{P}(\frac{M_j}{|S_j|} \geq \frac{1}{2})+\mathbb{P}(\frac{1}{2}>\frac{M_j}{|S_j|})
\end{align*}
Let's denote $\lambda = \frac{1}{2}-p > 0$. So we have $\frac{1}{2} = \lambda + p = \bar{p} - \lambda$ where $\bar{p} = 1-p$. Hence by Chernoff's bound the first term can be upper bounded by
\begin{equation*}
    \mathbb{P}(\frac{M_j}{|S_j|} \geq \frac{1}{2}) \leq \exp(-|S_j|\cdot D(p+\lambda||p))
\end{equation*}
and similarly the second term can be upper bounded by
\begin{equation*}
  \mathbb{P}(\frac{1}{2}>\frac{M_i}{|S_i|})  \leq \exp(-|S_i|\cdot D(\bar{p}-\lambda||\bar{p}))
\end{equation*}
Hence we have
\begin{align*}
  \mathbb{P}(Er^{(i)}) & \leq (k-1)[\max_{j\neq i}\exp(-|S_j|\cdot D(p+\lambda||p)) \\
  &+ \exp(-|S_i|\cdot D(\bar{p}-\lambda||\bar{p}))]
\end{align*}
Recall that from \eqref{pf:minS}, we have $\min_{i\in [k]}|S_i| \geq \frac{M\beta}{2}$ with probability at least $1-\frac{\delta}{4}$. Moreover, we observe that $D(0.5||p) = \min\{D(p+\lambda||p),D(\bar{p}-\lambda||\bar{p})\}$ by the symmetry of KL-divergence for Bernoulli distribution. Thus, the error probability for any new point can be upper bounded as
\begin{equation*}
  \mathbb{P}(Er) \leq \max_{i}\mathbb{P}(Er^{(i)}) \leq 2(k-1)\exp(\frac{-M\beta D(0.5||p)}{2})
\end{equation*}
Note that from Theorem~\ref{thm:1} we will need to query $\mathcal{Q}^{\ast}(\frac{\delta}{4})$ nodes in the aggressive search Phase if we want the exact result holds for probability at least $1-\frac{\delta}{4}$ in noiseless case. Hence by using the union bound, the error probability for exact recovery of these $\mathcal{Q}^{\ast}(\frac{\delta}{4})$ points is upper bounded by
\begin{equation*}
  2\mathcal{Q}^{\ast}(\frac{\delta}{4})(k-1)\exp(\frac{-M\beta D(0.5||p)}{2})
\end{equation*}
Requiring this to be smaller than $\frac{\delta}{4}$, then we have
\begin{equation*}
  M \geq \frac{2}{\beta D(0.5||p)}\log(\frac{8(k-1)\mathcal{Q}^{\ast}(\frac{\delta}{4})}{\delta})
\end{equation*}
This is exactly the forth requirement on $M$ in \eqref{thm:4req}. Further, via the union bound, the overall algorithm will succeed with probability at least $1-\delta$. Note that if we have exact recovery on these $\mathcal{Q}^{\ast}(\frac{\delta}{4})$ nodes, then we can indeed find the cut set $C$ by Theorem~\ref{thm:1}, which concludes the proof.

\end{document}